\documentclass{aamas2014}
\pdfoutput=1
\usepackage{hyperref}
\usepackage{url}
\usepackage{times}
\usepackage{helvet}
\usepackage{courier}
\usepackage{graphicx}
\usepackage{amsbsy}
\usepackage{booktabs}
\usepackage{color}
\makeatletter
\newif\if@restonecol
\makeatother

\usepackage[lined, ruled]{algorithm2e}
\usepackage{amsmath}
\usepackage{amssymb}
\usepackage{float}
\usepackage[override]{cmtt}
\usepackage{mathpartir} 

\newcommand \bluesemi {{\color{blue};}}

\newcommand \expeval[3] {{#1}\ \bluesemi\ {#2}\
  {\color{blue}\Downarrow}\ {#3}}

\newcommand \A \alpha

\title{A Programming Language With a POMDP Inside}
\numberofauthors{3}

\author{
\alignauthor
Christopher H. Lin\\
\affaddr{University of Washinton}\\
       \affaddr{Seattle, WA}\\
       \email{chrislin@cs.washington.edu}
\alignauthor
Mausam\\
\affaddr{Indian Institute of Technology}\\
\affaddr{Delhi, India}\\
\email{mausam@cse.iitd.ac.in}
\alignauthor
Daniel S. Weld\\
       \affaddr{University of Washington}\\
       \affaddr{Seattle, WA}\\
       \email{weld@cs.washington.edu}
}

\newcommand{\clowder}[0]{\textsc{Poaps}}

\newtheorem{thm}{Theorem}

\newcommand{\bi}{\begin{itemize}}
\newcommand{\ei}{\end{itemize}}
\newcommand{\BE}{\begin{enumerate}}
\newcommand{\EE}{\end{enumerate}}

\newcommand{\eg}{\mbox{\it e.g.}}

\newtheorem{theorem}{Theorem}
\newtheorem{lemma}[theorem]{Lemma}

\newcommand{\initab}{                           
\begin{tabbing}
XXX \= XXXX \= \kill
}
\newcommand{\begpub}{
\begin{quotation}
\noindent
}

\newcommand{\finpub}{
\end{quotation}
}

\begin{document}
\toappear{}
\maketitle
\begin{abstract}
We present \clowder, a novel planning system for defining Partially-Observable Markov Decision Processes (POMDPs) that abstracts away from POMDP details for the benefit of non-expert practitioners. \clowder\ includes an expressive adaptive programming language based on Lisp that has constructs for choice points that can be dynamically optimized. Non-experts can use our language to write adaptive programs that have partially observable components without needing to specify belief/hidden states or reason about probabilities. \clowder\ is also a compiler that defines and performs the transformation of any program written in our language into a POMDP with control knowledge. We demonstrate the generality and power of \clowder\ in the rapidly growing domain of human computation by describing its expressiveness and simplicity by writing several \clowder\ programs for common crowdsourcing tasks.  

\end{abstract}







\frenchspacing
\section{Introduction}
Although optimal decision-making is widely applicable across many aspects of human life, ultimately the ability to construct and use intelligent agents to make optimal decisions has been restricted to those who understand to some degree the theory of decision processes. Those who would greatly like access to such tools, like many in the crowdsourcing community \cite{weld-hcomp11} for example, must either resort to sub-optimal techniques that use approximate heuristics or hire a planning expert to formally define and solve their domain-specific problems.

This paper presents \clowder\ (Partially Observable Adaptive Programming System), which is a step toward bringing the power of decision theory to the masses. \clowder\ makes available the power of Partially-Observable Markov Decision Processes (POMDPs) to non-expert users through the interface of an adaptive programming language that provides an abstraction over POMDPs so that non-experts can write POMDPs without knowing anything about them. 

Like many previous approaches to proceduralizing decision processes\cite{andre-aaai02,boutilier-aaai00,pinto-icmla10}, we create a language that includes choice points, which allow the system to make optimal decisions adaptively. However, unlike previous approaches, we do not expect the programmer to explicitly reason about the state space or world dynamics. Such an expectation for our target users is impractical. In an informal experiment, we recruited scientists unfamiliar with artificial intelligence and introduced them to (PO)MDPs. We then asked them to define a state space for some simple crowdsourcing problems. However, all of them were unable to define a satisfactory state space, let alone an entire POMDP. In particular, they had trouble grasping the meaning and mathematical formalism of a POMDP hidden ``state.'' Therefore, the challenge is to create a language that can express all the details of POMDPs, yet hides these details from the programmer, but is still flexible enough to represent programs for a variety of scenarios. 

One of our key contributions is a division of work between experts and non-experts that achieves our desiderata. \clowder\ asks experts to define \emph{primitives}. A primitive consists of a function and a mathematical model of that function. The mathematical model describes some hidden state underlying the function and its dynamics. Non-expert programmers understand the primitives as procedures in terms of easily understood inputs and outputs and may not appreciate the hidden mathematical models. They can, however, compose the primitives into novel programs for their needs. \clowder\ then compiles their programs into completely new POMDPs.

For instance, suppose a user would like to write a program that uses crowdsourcing to label training data. An adaptive program that achieves this goal might be the following. For each datum, poll crowd workers for labels until the system is confident it can stop and return a label. For an adaptive program to make optimal decisions, it needs to both maintain some state that represents a current belief about what the correct label is, and know how to update this belief after every label observation. Instead of hiring a planning expert to handcraft a custom POMDP for this simple voting problem \cite{dai-aaai10, kamar-aamas12}, users, and in particular, non-experts, should be able to write a very simple program that abstracts away from state variables and probabilities: either ask another worker for another label and recurse, or return the label with the most number of votes. \clowder\ achieves this goal. Figure \ref{single-vote} shows a \clowder\ program for labeling (voting) that implements the algorithm we just described. It assumes there are two possible labels and reposes the problem as one of discovering if the first label is better than the second. Notice that the program makes no reference to any POMDP components in its definition. The user does not need to specify some hidden state that represents the correct answer. Instead, an expert has previously defined the primitive \texttt{crowd-vote}, which contains a mathematical model describing the dynamics of its inputs and ouputs, and we see that all the user needs to do in excess of providing the program logic is to use that primitive (by finding it in a library) and provide a choice point in the program. Then, \clowder\ will automatically determine the optimal branch to take at runtime. 

\begin{figure}
\center
\begin{tabbing}
\noindent\texttt{(d}\=\texttt{efine} \texttt{(vote-better? q a0 a1 c0 c1)}\\
\> \texttt{(c}\=\texttt{hoose} \\
          \>\>\texttt{(if } \= \texttt{(crowd-vote q a0 a1)}\\
          \>\>\>\texttt{(vote-better? q a0 a1 (+ c0 1) c1)}\\
          \>\>\>\texttt{(vote-better? q a0 a1 c0 (+ c1 1)))}\\
          \>\>\texttt{(if (> c0 c1) \#t \#f)))}\\
\end{tabbing}
\vspace{-15px}
\caption{A \clowder\ program for labeling that manages uncertainty without exposing it to the user. \texttt{q} is an input question, \texttt{a0,a1} are two possible answers, and \texttt{c0,c1} count the number of votes for each choice.}
\label{single-vote}
\vspace{-10px}
\end{figure}

We show the value of \clowder\ by writing many useful crowdsourcing POMDPs as \clowder\ programs. The simplicity of the language and programs leads us to believe that our system will be easily usable by non-experts. 

In summary, our contributions are (1) a system that exploits a separation of experts and non-experts that allows non-experts to write POMDPs while being isolated from the mathematical description of the decision process, and (2) an implementation that will likely help non-expert POMDP practitioners in their ability to write decision processes in a variety of domains. 

\section{Related Work}
Various languages have been proposed in the literature for representing POMDPs.  Several of those are declarative representations, which ask the user to explicitly declare each component (state, actions, etc.) of a POMDP. Examples include Cassandra-style format \footnote{http://pomdp.org}, Probabilistic PDDL \cite{younes-ipc04}, and RDDL \cite{sanner11}.  

Several procedural languages have also been developed including A$^2$BL \cite{simpkins-oopsla08}, ALisp \cite{andre-aaai02} and concurrent ALisp \cite{marthi-ijcai05}, Hierarchical Abstract Machines (HAMs) \cite{parr-nips98}, and Programmable Hierarchical Abstract Machines (PHAMs) \cite{andre-nips01}. All these representations allow a user to provide control knowledge in the form of a procedural program for an {\em existing} and explicitly specified MDP. In other words, when one writes an ALisp program, one must additionally explicitly specify an MDP that the program is tied to and constraining. 

DTGolog \cite{boutilier-aaai00} is a situation calculus-based procedural language that can both define decision problems (by defining a set of axioms) and specify control. While this unification is useful for experts, users must still explicitly specify MDPs before they can write control programs. While non-expert users can write control policies for expert-written MDPs, they cannot write their own MDPs. Additionally, they must explicitly use the components of the MDP in their control policies.

Stochastic Programs (SP) \cite{mcallester-99} provide a language for experts to write world models and primitive actions and then compose those primitive actions to create control policies for the corresponding world model. While they do no consider non-expert use of the language, one can imagine non-experts using the primitive actions to create control policies. However, these non-expert users must explicitly use POMDP components. In particular, the primitives do not abstract away from the state space; they take state variables as arguments, and return state variables and observations. SP also does not allow for learning policies. The language only provides for specifying a complete control policy and evaluating its utility.

The work on adaptive programs \cite{pinto-icmla10} allows non-expert users to quickly construct observable decision processes by writing programs that can contain optimizable choice points. However, construction of POMDPs still requires the user to explicitly model POMDP details like belief states and belief updates. 

The key difference between our work and all the related work is that we do not ask the programmer to explicitly define or use POMDP components. Instead we leverage a division of work between experts and non-experts whereby non-experts can glue together expert-provided POMDP components to create entirely new POMDPS for their own problems.

\section{Crowdsourcing Background}


We are motivated by and describe our system using many examples taken from the crowdsourcing literature. Crowdsourcing \emph{requesters}, those who hire crowdsourced workers, often design \emph{workflows} to complete their tasks. An example of a simple workflow is the labeling workflow we described in the Introduction (Figure \ref{single-vote}).  Another example of a workflow is the iterative improvement workflow \cite{little09}. Suppose a requester wants to generate some artifact, like a text description of an image. In the iterative improvement workflow, he first hires one worker to improve an existing caption. Then, he asks several workers to vote on whether the new caption is better than the old caption. Finally, he repeats this process until he is satisfied with the caption. Previous work has hand-crafted a POMDP for this particular workflow in order to make dynamic decisions like when to vote and when to stop, showing significant savings over static policies \cite{dai-aaai10}. Our system would allow these requesters, who are likely not planning experts, to easily write a program for this workflow (and others) that implicitly defines a POMDP, which our system can then optimize and control.  


\section{Primitives}
In order to interpret a program like the one for iterative improvement as a POMDP, \clowder\ needs mathematical models for function calls, like the one that hires a worker to improve an artifact. \clowder\ asks experts to define primitives to bootstrap this process. A \emph{primitive} is a ten-tuple $\langle {\cal D}, {\cal R}, \Omega, {\cal T}, {\cal O}, {\cal I}, {\cal C}, {\cal D}_U, {\cal R}_U, {\cal F} \rangle$, where:
\begin{itemize}
\item ${\cal D} =  {\cal D}^1 \times \ldots \times {\cal D}^n$ is a set of \emph{domain states}.
\item ${\cal R}$ is a set of \emph{range states}.
\item $\Omega$ is a set of \emph{observations}.
\item  ${\cal T}: {\cal D} \times {\cal R} \rightarrow [0,1]$ is a \emph{transition function}.
\item ${\cal O}: {\cal R} \times \Omega \rightarrow [0,1]$ is an \emph{observation function}.
\item ${\cal I}$ is an $n$-dimensional indicator vector indicating which of the ${\cal D}^i$ are observable.
\item ${\cal C}: {\cal D} \rightarrow \mathbb{R^+}$ is a \emph{cost function}.
\item ${\cal D}_U =  {\cal D}_U^1 \times \ldots \times {\cal D}_U^n$ is a set of \emph{user domain states}.
\item ${\cal R}_U$ is a set of \emph{user range states}.
\item  ${\cal F}: {\cal D}_U \rightarrow {\cal R}_U$ is a \emph{user function}. 
\end{itemize}

An expert defines all 10 of these components. Intuitively, a primitive is a function (the user function ${\cal F}: {\cal D}_U \rightarrow {\cal R}_U$), and a model of that function (the rest of the components). We note that in the special case when ${\cal D} = {\cal R}$, $\langle {\cal D}, {\cal R}, \Omega, {\cal T}, {\cal O}, {\cal C} \rangle$ is a one-action POMDP. The best way to understand the purpose of primitives is through an example. In particular, we discuss how we would define the primitive \texttt{c-imp}, which would be a function used in iterative improvement to improve an artifact.

First, we define the function part. \texttt{c-imp} should take an artifact $\alpha$ as input, call some crowdsourcing API, and return an improved artifact $\alpha'$. Therefore, we define the user function to be ${\cal F}(\alpha \in {\cal D}_U)$ = calltoAPI($\alpha$), where ${\cal D}_U = {\cal R}_U$ are the set of all possible artifacts $\alpha$ (\eg, the set of all strings). Now we want to define the model part of the primitive, which will track the quality of artifacts being input and output by ${\cal F}$. We define ${\cal D} = {\cal R} = [0,1]$ to represent the hidden quality of the artifact. The transition function ${\cal T}$ needs to encode the probability of getting an artifact of quality $q'$ if a worker improves an artifact of quality $q$. Therefore, we define ${\cal T}(q\in {\cal D},  q' \in {\cal R}) = P(q' | q)$ using some conditional distribution like a Beta distribution. We set ${\cal C}$ to be the amount of money paid to a worker, which can be some constant like 5 cents.  \texttt{c-imp} produces no observations so $\Omega$ is empty, and hence there is no observation function ${\cal O}$. Finally, we set ${\cal I} = (0)$ indicating that ${\cal D}$ is not observable.

This primitive combines a model for the improvement of an artifact with a function that outputs an improvement of the artifact. We can view each ${\cal D}^i_U$ as a model for ${\cal D}^i$ and ${\cal R}_U$ is a model for ${\cal R}$.  So, for a non-expert user who does not care about or cannot understand the model, a primitive is simply the function  ${\cal F}: {\cal D}_U \rightarrow {\cal R}_U$. These non-experts can call primitives in their programs, and when they do so, they expect they are calling the function ${\cal F}$.
\section{The Language}
We now describe the \clowder\ language, which users use to express adaptive programs using primitives. We define a \clowder\ program to be a function definition written in the \clowder\ language. The \clowder\ language is an extension of Lisp, because Lisp is both easy to write and easy to interpret. Following previous work \cite{andre-aaai02,boutilier-aaai00,pinto-icmla10}, we add the special form \texttt{(choose <$exp_0$> <$exp_1$> ... )}. 

The \texttt{choose} special form is a construct for dynamic execution. It takes a variable number of arguments, each of which is a Lisp S-expression. When used in a program, it describes a choice point in the program, meaning that at runtime, \clowder\ will dynamically decide the optimal argument expression to execute. 

A key contribution of \clowder\ is how function calls are interpreted. However, we first emphasize that for a non-expert user, \clowder\ behaves just as an ordinary programming language. When a non-expert user calls a primitive: $(p\ \ arg_0\ \ldots arg_n)$, the expression evaluates to ${\cal F}(arg_o, \ldots, arg_n)$ where $arg_i \in {\cal D}^i_U$.   A function call is just a function call, regardless of whether the function is a \clowder\ primitive or a user-defined function. Figure \ref{c-imp} shows a \clowder\ program that a crowdsourcing expert might write for improving a piece of text using crowdsourcing. It is a simplified version of iterative-improvement that removes voting. 

\begin{figure}
\begin{tabbing}
\noindent\texttt{(d}\=\texttt{efine} \texttt{(improve text)}\\
\> \texttt{(c}\=\texttt{hoose} \\
          \>\>\texttt{(improve (c-imp text))}\\
          \>\>\texttt{text)))}
\end{tabbing}
\caption{A \clowder\ program for improving a piece of text. \texttt{text} is the current text.}
\label{c-imp}
\end{figure}

To the non-expert user, the argument \texttt{text} is bound to a string, $\alpha$. During execution, there are two execution paths. Suppose the program chooses the first path.  When the string, $\alpha \in {\cal D}_u$, is passed to \texttt{c-imp}, a primitive we described in the previous section, a function, ${\cal F}$ is called to hire a crowdworker to improve the string. \texttt{c-imp} returns the improved string $\alpha' \in {\cal R}_U$ and the program recurses. If the program chooses the second execution path, the string is returned.

However, the semantics of the \clowder\ language are more complex. The expert user understands that in \clowder, all variables are actually bound to \emph{two} values, and thus all expressions evaluate to two values. The first value, the \emph{Normal value}, is the usual value that the non-expert user sees and understands, and is the same as it would be in any other programming language. For example, \texttt{text} is bound to a string. The second value is a \emph{Poaps value} that can be unobservable, and hence, represented by a distribution in our system. This value is the value of a state variable in the POMDP \clowder\ constructs.  Let $\overline{exp}$ represent this possibly hidden \clowder\ value of some expression, $exp$.

For the expert user, calling a primitive is everything that it is for the non-expert user. However, the expert user knows that in addition to being bound to the Normal value, the result of an expression $(p\ \ arg_0\ \ldots arg_n)$ is bound to a \clowder\ value $r \in {\cal R}$ with probability ${\cal T} ((\overline{arg_0} \ldots, \overline{arg_n}), r)$, where $\overline{arg_i} \in {\cal D}^i$. Furthermore, when $p$ is called, an observation $o \in \Omega$ is produced  with probability ${\cal O}(r, o)$. The \clowder\ agent reasons about the \clowder\ values in the program using observations in order to make decisions. 

Consider the program in Figure \ref{c-imp}. The argument \texttt{text} is actually bound to two values. The first value, the string, is what the programmer cares about. The second \clowder\ value can be thought of as some unobservable measure of the quality of the text $q \in [0,1]$. The domain of this second value was implicitly specified by an expert when he defined the primitive \texttt{c-imp}. When \texttt{c-imp} is called, in addition to returning an improved string, a \clowder\ value $q' \in [0,1]$ is also returned with probabilities defined by ${\cal T}$. Then, the program recurses and \texttt{text} is now bound to both the new string and $q'$. In this example, no observation is produced.

We emphasize that the expert, the program, and the \clowder\ agent, may \emph{not} know the \clowder\ values. The \clowder\ values may be unobservable, so the best an expert and an agent can do is hold a belief about what they might be, using the observations as hints. Therefore, the next step in \clowder\ is to compile a \clowder\ program into a POMDP, and then solve the POMDP to generate a policy that controls the program based on the beliefs about the hidden values of the variables.

As another example, we provide a description of the voting program we present in Figure \ref{single-vote}. The program uses three primitives: \texttt{+}, \texttt{>}, and \texttt{crowd-vote}. The \clowder\ values of \texttt{q}, \texttt{a0}, and \texttt{a1} can be thought of as unobservable measures of the difficulty of the question and the quality of the two answers, respectively. The \clowder\ values of \texttt{c0} and \texttt{c1} are observed, and are the same as their Normal values. \texttt{crowd-vote}'s range states are the same as its observations (${\cal R} = \Omega$), and its observation function is defined as $O(r\in{\cal R}, \omega \in \Omega) = 1$ if and only if $r = \omega$. So, when it is called, it returns a \clowder\ value with probability defined by ${\cal T}$, and the observation it emits is the same \clowder\ value. The Normal value it returns is also the same as the \clowder\ value. The primitives \texttt{+} and \texttt{>} are defined in the expected way. \footnote{We note here that for planning purposes, the \texttt{if} construct in our language uses the \clowder\ value of its test expression to determine which branch to take instead of the Normal value. The next section will show how during execution we insert observations to tell our agent what branch was actually taken.} 

Of course, \clowder\ programs do not restrict users to calling primitives. Users can also call their own user-defined functions. For example, they can use their program for voting (Figure \ref{single-vote}) in a program for iterative-improvement (Figure \ref{it-i}). We note that in the program for voting, the operators \texttt{>} and \texttt{+} are primitives. When a function calls another user-defined function, the semantics are ``call-by-poaps-value.'' Quite simply, in contrast to the normal ``call-by-value'' semantics where only one value is copied and passed, in our language, both the normal value and \clowder\ value are copied and passed.

\begin{figure}
\begin{tabbing}
\noindent \texttt{(d}\=\texttt{efine} \texttt{(it-i image worse-text better-text)}\\
\> \texttt{(c}\=\texttt{hoose}\\
\>\>\texttt{(it-i} \= \texttt{image} \= \texttt{better-text}\\
\>\>\>\> \texttt{(c-imp  better-text))}\\
\>\>\texttt{(if} \= \texttt{(vote-better? }\= \texttt{image better-text}\\
\>\>\>\> \texttt{worse-text 0 0)}\\
\>\>\>\texttt{(it-i image worse-text better-text)}\\
\>\>\>\texttt{(it-i image better-text worse-text))}\\
\>\>\texttt{better-text))}\\
\end{tabbing}
\caption{A \clowder\ program for iterative-improvement on descriptions for images.}
\label{it-i}
\end{figure}

We now define a compiler for the \clowder\ programming language, which converts the language into a POMDP.

\section{The Compiler}
Before we delve into the technical details of the compiler, we provide a high-level description of the process.

The whole point of converting a \clowder\ adaptive program into a POMDP is to enable construction of an optimal policy for the program, but this requires an optimality criterion. Since optimality is different for every user, we need the flexibility to construct different utility functions or goals for individual users. 
In light of these challenges, we assume that executing a primitive incurs a cost defined by the primitive, but that an externally-provided and expert-defined mechanism for goal or utility elicitation (\eg ~\cite{chajewska-aaai00}) 
is used to guide the overall program objective.  For example, consider the voting program of Figure \ref{single-vote}. It might cost \$0.05 to execute the \texttt{crowd-vote} primitive, but learning a given user's desired target accuracy in order to guide the execution of the program requires additional information. A reasonable goal elicitation module for a POMDP compiled from this program is one that simply asks the user for a desired accuracy and budget, and converts the desired accuracy into a goal belief state and the budget into a horizon to ensure no dead ends. Such a goal elicitation module could be used for any program that outputs ``correct answers'' and uses \texttt{crowd-vote}.   

Alternatively, users can forego providing their goals/utilities with an external mechanism and simply write goals into their programs. For example, they can simply write their own termination conditions that rely only on the visible parts of their programs. Whether or not we have elicited goals/utilities, our goal is to execute the branches that minimize the expected sum of costs.

From our description of the \clowder\ language, we have a very natural, but unbounded, decision process that emerges. This decision problem can be posed as a history-based MDP. The state of the MDP consists of all the branches taken and observations received so far. An action in the MDP is choosing a branch in the program. Taking an action produces observations and costs, so the transition function (from a list of actions and observations to another list of actions and observations) is completely determined by the dynamics of the underlying primitives and our ``call-by-poaps-value'' semantics.

However, we do not want to define such an MDP because solution methods will not scale. Instead we define an equivalent POMDP. We now define the \clowder\ compiler that produces this POMDP by describing in detail the process that converts any \clowder\ program into a POMDP. Given an input program $p$, the compiler converts $p$ into a POMDP $(M \circ S) (p)$ by the following steps:
\begin{enumerate}
\item Define a set of states $S(p)$ by statically analyzing $p$. Each state variable of $S(p)$ will represent the \clowder\ value of some variable or expression in $p$ or a function called by $p$. So, we call a state $c \in S(p)$ a \emph{control state}, because it is the part of the state that determines what action should be taken.
\item Construct a Hierarchical Abstract Machine (HAM) \cite{parr-nips98} $M(p)$ by evaluating the program under a set of operational semantics. A HAM is a type of nondeterministic finite state machine. Each state $m \in M(p)$ will be a representation of the current program counter. In other words, it tells the agent where in the program it currently is. So, we call this state the \emph{machine state}. This state will be fully observable, and provides information about what actions are available to take.
\item Following the insights of \cite{parr-nips98}, merge $M(p)$ with $S(p)$ to create a POMDP $(M \circ S) (p)$ with state space $\hat{S} = StatesOf(M(p)) \times S(p)$, and define the actions, transition function and observation function of $(M \circ S) (p)$ by traversing $M(p)$ and applying a set of rules. Therefore, a single state in our POMDP will be a tuple $(m, c)$, where one part of the state is the machine state, and the other part of the state is the control state. Finally, using a separate goal/utility elicitation module, integrate the goal or rewards into the POMDP.
\end{enumerate}

\subsection{Step 1: Creation of a State Space S}
First, we need to define the state variables for the arguments of $p$. Let $X(p) = \{arg_1, \ldots, arg_n\}$ be the set of all arguments of $p$. In order to construct $S(p)$, the compiler needs to know the state space of each argument. The state space of an argument $arg_i$ is defined by the domain state space ${\cal D}^i$ of the primitives that use $arg_i$ \footnote{We assume that all the primitives that use a variable $arg_i$ have the same state space ${\cal D}^i$. We can relax this assumption by using typing techniques.}. 

Next, we need to define state variables for all subexpressions in $p$. 
A program $p$ in our language can be viewed as an evaluation tree of expressions. For example, Figure \ref{improve-tree} shows the corresponding tree for the \texttt{improve} program (Figure \ref{c-imp}). In order to remember all state necessary to control, we have a state variable for each subexpression in $p$. We denote this state space $R(p)$.

\begin{figure}
\centering
\includegraphics[scale=0.15]{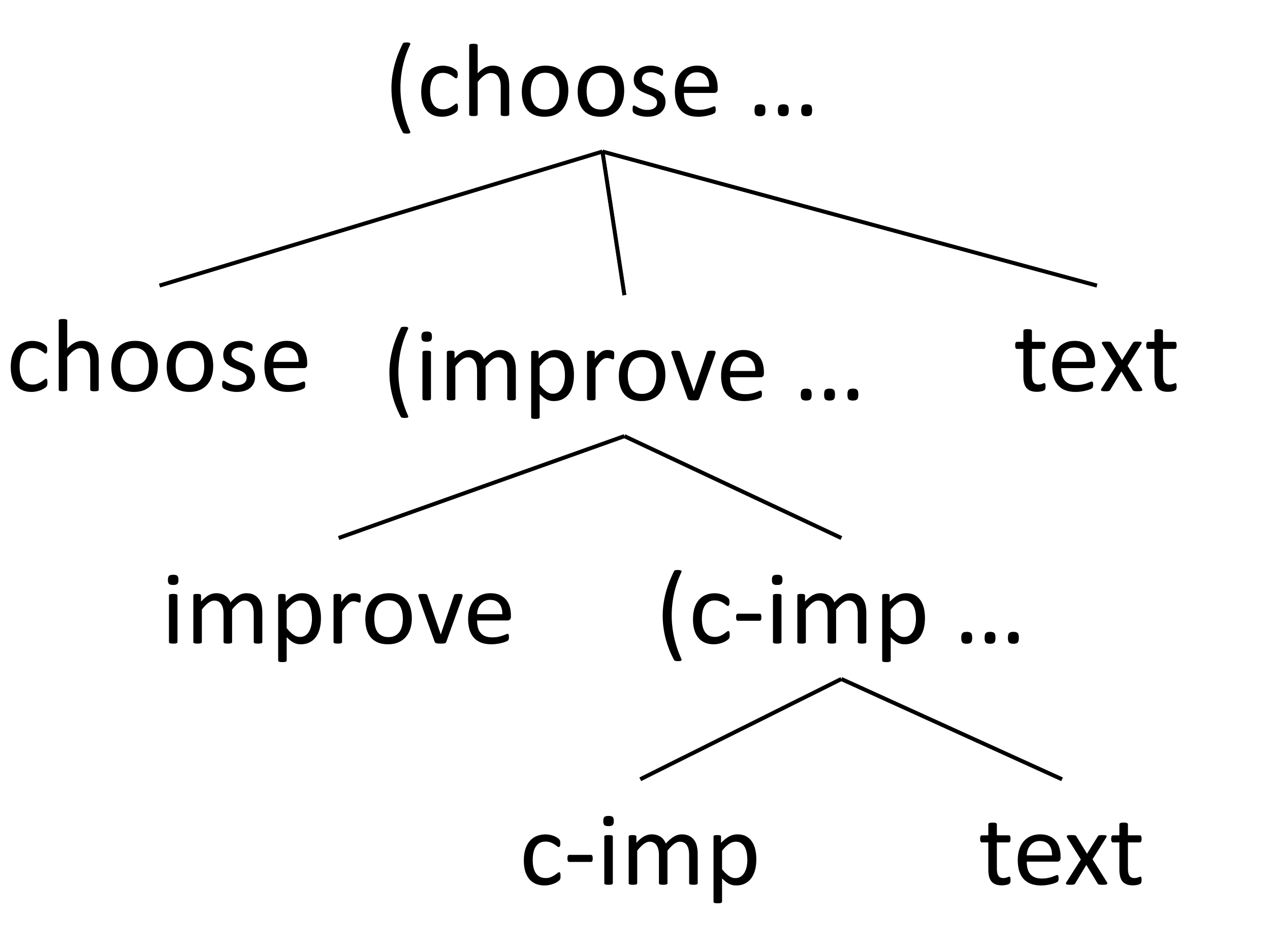}
\caption{The tree for \texttt{improve}}
\label{improve-tree}
\end{figure}



Let $F(p)$
be the set of \clowder\ programs corresponding to user-defined functions that are called in $p$. Then, we abuse notation for ease of understanding, and recursively define $S(p) = S(F(p)) \times Domain(X(p)) \times Domain(R(p))$. Thus, the state space that we have constructed is a cross product of the state spaces of all the functions that $p$ calls, the state spaces of all the arguments of $p$, and the state space which consists of all possible evaluations of every expression in $p$. This state space is the control state space. 
Since we use Monte-Carlo solution methods to solve our POMDPs, we do not need to express the state in a closed, non-recursive form. 

\subsection{Step 2: Construction of a HAM}
The second step in the compilation process is to construct the machine state space by constructing a HAM \cite{parr-nips98} $M(p)$ given $p$ and $S(p)$. The HAM's states will be used in our constructed POMDP as observable state variables that represent the current program counter. Each HAM state represents the evaluation of an expression. In other words, the HAM will be the part of the POMDP that says where in the evaluation tree we are for a program $p$. 

The five types of states of a HAM are Action, Call, Choice, Start, and Stop. \emph{Call} states represent a call to a user-defined function. They will execute the corresponding HAM. \emph{Choice} states can transition to one of many HAMs. \emph{Stop} states signify the end of execution of a HAM and return control to the next state of the parent calling HAM. \emph{Start} states denote the initial HAM state. \emph{Action} states represent the evaluation of a symbol or constant, or the execution of a primitive.  

Finally, we add a sixth type of state: an Observation State. \emph{Obs} states do not represent the evaluation of any expression in $p$. These states will do nothing except emit an observation. These states are inserted after conditionals so that an agent can eliminate inconsistent beliefs. These states were not necessary in \cite{parr-nips98} because their world was fully observable.

We evaluate the program $p$ to a HAM by using \emph{inference rules} in the same way computer programs are evaluated by interpreters. We recursively evaluate subexpressions to HAMs using inferences rules and then combine them into larger and larger HAMs for each parent expression until we have a HAM for $p$. 

Consider the \texttt{improve} program in Figure \ref{c-imp}. We first use an inference rule for \texttt{define}, which leads to using a \texttt{choice} inference rule. We provide a simplified version of the \texttt{choice} inference rule here. 

\begin{mathpar}
\infer[choice]{\expeval{H}{e_i}{M_i}}{\expeval{H}{(choose\ e_1\ \cdots e_n)}{Choice(M_1,\ldots,M_n)}}
\end{mathpar}

The rule says that if each expression $e_i$ evaluates to a HAM $M_i$ under the heap $H$, then the expression $(choose \ e_1 \cdots e_n)$ evaluates to a HAM that contains a Choice node that can transition to any of the HAMs $M_i$. Thus, for the \texttt{improve} program, when we evaluate the \texttt{(choose...)} subexpression, there are two expressions that needs to be recursively evaluated. 
The result is the HAM in Figure \ref{c-imp-ham}.

After we construct a HAM, we post-process by adding a Start state to the beginning and a Stop state to the end. Additionally, if we see any tail calls (Call states that are leaf nodes), we can add an edge from the call state to the beginning of the HAM, and change the semantics of that call state so that it transitions to the next state instead of executing another HAM as a subroutine.

\begin{figure}
\centering
\includegraphics[scale=0.25]{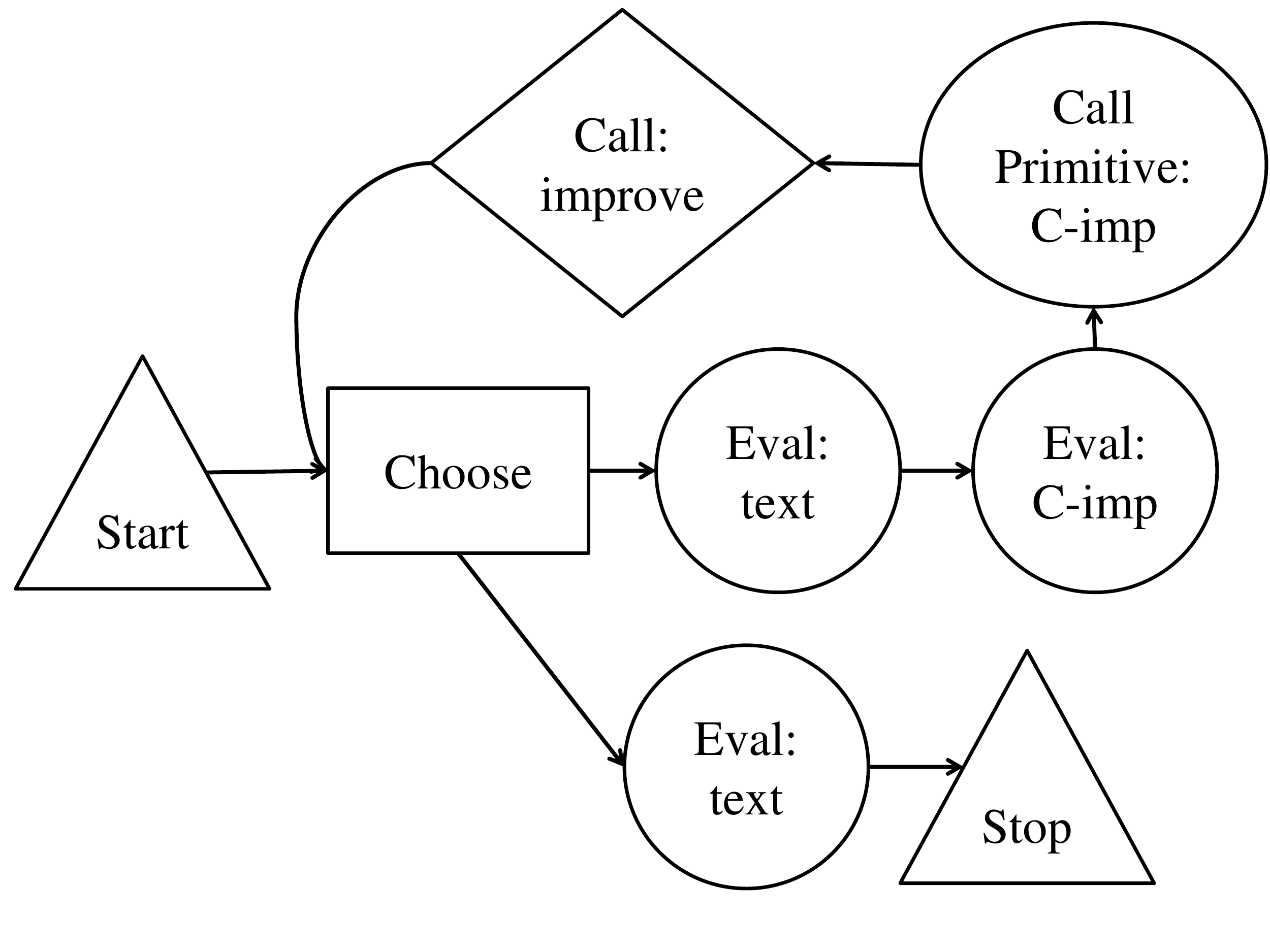}
\caption{A HAM for the program \texttt{improve} in Figure \ref{c-imp}. Circles are action states, diamonds are call states, rectangles are choose states.}
\label{c-imp-ham}
\end{figure}

\subsection{Step 3: Putting It All Together}
Letting $S(M(p))$ denote the set of states of a HAM $M(p)$, the state space of the POMDP, $(M \circ S) (p)$, that we construct is $\hat{S}(p) = S(p) \times StatesOf(M(p))$.

The actions depend only on the current machine state $m \in M(p)$, which is fully observable. In any machine state that is not a Choice state, the agent can only take one action. If the machine state is a Choice state, then the actions are the branches of the Choice state.
 
We define the transition function $T(s, a, s')$ of  $(M \circ S)(p)$ such that values are passed around correctly between states to enforce call-by-poaps-value semantics. 

The observation function $O(s', o)$ is simple. Observations are only received in two scenarios. First, observations can be received when executing a primitive and they are defined by the primitive.  
Second, observations can be received when transitioning to a HAM observation state. 

\section{Optimality}


We now show that the POMDP we construct is correct, in that its optimal policies result in the optimal executions of its corresponding program. 
\begin{lemma}
For any POMDP $(M \circ S)(p)$ for a program $p$, let $C$ be the set of choice belief states, which are the belief states in which the machine component of every possible world state is a choice node. There exists a semi-MDP  $(m \circ s)(p)$ with state space $C$, such that an optimal policy $\pi$ for $(m \circ s)(p)$ corresponds to an optimal policy $\Pi$ for  $(M \circ S)(p)$, in that $\Pi$ simply augments $\pi$ by mapping belief states not in the domain of $\pi$ to their single, default actions.
\end{lemma}
\begin{proof}
Consider the belief-MDP that corresponds to $(M \circ S)(p)$. Consider the states that are not choice beliefs. We can remove these states to produce an equivalent belief-Semi-MDP. The optimal policy for this belief-Semi-MDP is the same as the optimal policy for the belief-MDP, and thus the same as the optimal policy for $(M \circ S)(p)$.
\end{proof}
\begin{thm}
Let $\mathcal{M}$ be the history-based MDP associated with a program $p$. Then an optimal policy for the POMDP we generate, $(M \circ S)(p)$, can be used as an optimal policy for  $\mathcal{M}$.
\end{thm}
\begin{proof}
$\mathcal{M}$  and $(m \circ s)(p)$ are stochastically bisimilar \cite{givan-aij03} (we map each history to its corresponding belief state), and the optimal policy of $(m \circ s)(p)$ corresponds to that of $(M \circ S)(p)$ (Lemma 1), so the optimal policy for $(M \circ S)(p)$ can be used as an optimal policy for $\mathcal{M}$ (Lemma 1).
\end{proof}
This theorem also affirms that an optimal policy for our constructed POMDP is not just a ``recursively optimal'' \cite{dietterich-jair00} policy. For example, suppose a user has written a program $p$ that calls some other user-defined program $f$ that also calls some other user-defined program $g$. Then, the optimal policy's actions while in $g$ consider not only the state local to $g$, but also the state local to $f$ and the state local to $p$. In other words, an optimal policy for $(M \circ S)(p)$ does not solve lots of primitive POMDPs in isolation. Notably, a POMDP solver that produces an optimal policy will update beliefs about state variables local to $p$ even when observations are made about correlating state variables local to $g$.  



\section{Monte Carlo Planning}
We solve the POMDP when a user runs a program. The POMDP that we construct can potentially have many unreachable states. Additionally, we do not want to construct the entire state space or the full matrix representation of the transition and observation functions, since these can be very large or infinite. Therefore, we choose to use online Monte-Carlo methods to solve the POMDP.  While we try using a UCT-based solver, POMCP (without a rollout policy) \footnote{Since \clowder\ is a general representation language, the existence of a rollout policy cannot be assumed.} \cite{silver-nips10}, we find that the value function can take an extraordinary amount of time to converge.

Instead, we modify RTDP-Bel \cite{barto-ai95} to create C-RTDP, an algorithm similar to HAMQ-learning \cite{parr-nips98} that takes advantage of the fact that the actual complexity of the POMDP is determined by the number of choice points. C-RTDP modifies RTDP-Bel by only performing backups on \emph{choice beliefs}. In the POMDP that we construct, all the states that have non-zero probability in a reachable belief state will always have the same machine state. A \emph{choice belief} is one in which the machine component of every state is a choice node.

We fully specify our algorithm below. $b_a^o$ means the belief state given that the agent had belief state $b$, then took action $a$ and received observations $o$. $C_a^o$ is the expected cost of taking action $a$ in belief $b$ and receiving observations $o$. 

\begin{algorithm}
\DontPrintSemicolon
Initialize belief $b = b_0$\;
Sample $s = (m, c)  \sim b$\;
\If {$m$ is terminal} {
  Return \;
}
\eIf {$m$ is a choice node} {
\For{every action $a$} {
$\Psi = \cup_{i\in \mathbb{N}}\{(o_1,\ldots,o_i) | b_a^{o_1,\ldots,o_i}$ is choice belief$ \}$ \;
$Q(a, b) = \sum_{\psi \in \Psi} P(\psi | a)(C_a^\psi + V(b_a^\psi))$\;
$\hat{a} = \min_a Q(a, b)$\;
}
} 
{
$\hat{a} = $ default action \;
}
Update $V(b) = Q(\hat{a}, b)$\;
Sample $s' \sim T(s, \hat{a}, s')$, $o \sim O(\hat{a}, s')$\;
Update $b_0 = b_{\hat{a}}^o$ and repeat \;
\caption{C-RTDP (One simulation)}
\label{crtdp}
\end{algorithm}

We can show that C-RTDP converges to the optimal policy:

\begin{thm}
For any POMDP $(M \circ S)(p)$ for a program $p$, C-RTDP will converge to the optimal policy with probability 1.
\end{thm}
\begin{proof}
For a POMDP  $(M \circ S)(p)$, C-RTDP solves the corresponding belief-Semi-MDP. Then by Lemma 1, we have that C-RTDP solves $(M \circ S)(p)$.
\end{proof}

By using a Monte-Carlo, online approach, we gain several advantages. Initializing the initial belief is easy. Suppose a user runs a program for a function f with arguments $arg_0\, \ldots, arg_n$. None of the expressions have been evaluated yet, so we only need to initialize our belief of the arguments $arg_i$. If $arg_i$ is observable, as defined by the primitives that use it inside $f$, then we simply define the \clowder\ value of $arg_i$ to be equal to the Normal value. Therefore, if a state is observable, its space can be infinite. If $arg_i$ is unobservable, then we initialize a uniform belief over the state space defined by the primitives that use it.

\section{Proof-of-Concept}
As a proof of concept that \clowder\ can run programs, we implement the \clowder\ system and write the voting program from the introduction. We also implement a goal eliciation module. The primitive \texttt{crowd-vote} has a cost of 1 cent and we specify a goal accuracy of 90\%. We run the program on Mechanical Turk with 1100 named entity recognition tasks. Each named entity recognition task begins by providing the worker with a body of text and an entity, like ``\textbf{Washington} led the troops into battle.'' Then it uses \emph{Wikification}~\cite{milnewitten08,ratinov11}
to find a set of possible Wikipedia articles describing the entity, such
as ``Washington (state)'' and ``George Washington,'' and asks workers to choose the article that best describes the entity. \clowder\ achieves an overall accuracy of 87.73\% with an average cost of 4.33 cents. This result is consistent with those in \cite{dai-aij13}, showing that our general purpose implementation can perform at par compared to an expert-written problem-specific POMDP, suggesting the value of our system to end-users.

\section{A Larger Example}
Throughout this paper, we have expressed several practical crowdsourcing problems in our language. Note that all our programs have only used two non-trivial expert-defined primitives: \texttt{c-imp} and \texttt{c-vote} (trivial Lisp-primitives like \texttt{+} and \texttt{>} whose hidden behaviors are identical to their visible behaviors come with \clowder). We demonstrate the versatility of our paradigm by writing \emph{find-fix-verify} \cite{bernstein10}, a more complex and popular workflow that can be used for crowdsourcing edits to text. 

For example, given a piece of text as input (like a term paper), it first asks crowdworkers to find patches in the text that need attention. Then, given these subsets of text, it asks workers to revise them. Finally, given the revisions, it asks workers to verify that the revisions are better, through some voting mechanism.

In Figure \ref{ffv}, we show the \clowder\ program for find-fix-verify (\texttt{ffv}). In addition to using \texttt{c-imp} and \texttt{c-vote}, we only need to use one more non-trivial expert-defined primitive: \texttt{c-find}, which asks a worker to provide an interval of text that requires attention. However, note that we can potentially use only \texttt{c-imp} and \texttt{c-vote} by eliminating \texttt{c-find} and replacing it with \texttt{c-vote} where the possible answers are a set of intervals. These primitives provide all the information we need to construct the POMDP for the program. The domain state spaces of the primitives provides the state spaces of the \clowder\ values of the arguments to those primitives, and the transition functions describe the probabilities of the \clowder\ values of the returns. 

 The program also uses trivial string and list manipulation primitives like \texttt{get-relevant-text}, \texttt{replace-text}, and \texttt{merge}, which like \texttt{+} and \texttt{>}, do not need to be expertly defined. \texttt{worse-text} represents what we think is the worse-text and \texttt{better-text} represents what we think is the better text. We call \texttt{ffv} with \texttt{worse-text} bound to an empty string and \texttt{better-text} bound to the text we want to improve.

There are three choices. We can: 1) find mistakes in and fix the better text (\texttt{find-fix} and then recurse, or 2) verify which version of the text is better and recurse, or 3) return the better text. \texttt{find-fix} first calls \texttt{find} to repeatedly ask workers to provide intervals in the text that require work. Then, it calls \texttt{fix} with those intervals and repeatedly asks workers to improve the text in those intervals. 

A simple goal elicitation module for this program could simply ask the user whether they would like a ``Almost-Perfect'' text, an ``Excellent'' text, or a ``Satisfactory'' text. It would translate the choice into a goal belief on the hidden quality of the text, and then minimize the expected cost of achieving that goal.



\begin{figure}
\center
\begin{tabbing}
\noindent \texttt{(d}\=\texttt{efine} \texttt{(ffv worse-text better-text)}\\
\> \texttt{(c}\=\texttt{hoose}\\
\>\>\texttt{(ffv} \= \= \texttt{better-text}\\
\>\>\>\> \texttt{(find-fix better-text))}\\
\>\>\texttt{(if} \= \texttt{(vote-better? }\= \texttt{'which is better?' better-text}\\
\>\>\>\> \texttt{worse-text 0 0)}\\
\>\>\>\texttt{(ffv worse-text better-text)}\\
\>\>\>\texttt{(ffv better-text worse-text))}\\
\>\>\texttt{better-text))}\\
\end{tabbing}
\begin{tabbing}
\noindent\texttt{(d}\=\texttt{efine} \texttt{(find-fix text)}\\
\>\texttt{(fix text (find text '()))}
\end{tabbing}
\begin{tabbing}
\noindent\texttt{(d}\=\texttt{efine (fix text intervals)}\\
\>\texttt{(ch}\=\texttt{oose}\\
\>\>\texttt{(let } \=  \texttt{((next-int (choose intervals))}\\
\>\>\> \texttt{ (next-text (get-relevant-text text next-int))}\\
\>\>\> \texttt{ (better-text (c-imp next-text)))}\\
\>\>\> \texttt{(fix }\=  \texttt{(replace-text text next-int better-text)}\\
\>\>\>\> \texttt{intervals))}\\
\>\>\texttt{text))}\\
\end{tabbing}
\begin{tabbing}
\noindent\texttt{(d}\=\texttt{efine (find text intervals)}\\
\>\texttt{(ch}\=\texttt{oose}\\
\>\>\texttt{(find text (merge (c-find text) intervals))}\\
\>\>\texttt{intervals))}
\end{tabbing}
\caption{A \clowder\ program for the find-fix-verify workflow.}
\label{ffv}
\end{figure}

For even more examples of programs we can write, please refer to the Appendix.

\section{Conclusion}
We have presented \clowder, a system that provides a language for writing decision processes that provides an abstraction over POMDPs. Knowledge of POMDPs is not a prerequisite for being able to use decision-theory in everyday applications. In particular, the states and dynamics of POMDPs are hidden from users. We have shown how crowdsourcing experts can use \clowder\ to build and optimally control many of their workflows. We have also implemented \clowder\ and conducted a proof-of-concept experiment that shows that \clowder\ can run the voting program of Figure \ref{single-vote} and achieve results comparable to an expert-written problem-specific POMDP. The complete \clowder\ system that we have built will be available at the authors' websites.

\section{Future Work}
Our work on \clowder\ is just the beginning. We imagine many future directions:

1) A key question is whether or not \clowder\ is easy to use by people who are not planning experts. A comprehensive answer to this question requires a user study of our complete system.

2) \clowder\ does not allow access to underlying POMDP details through its language. But ideally, we would like to allow users to modify whatever aspects they understand (\eg, a subset of state variables or costs). We imagine an extension of \clowder\ that allows users to work with POMDP/primitive components in their programs.

3) \clowder\ allows users to specify \emph{hard constraints} on policies. By writing an adaptive program, they are exactly specifying the policies that may be chosen by a planner. In other words, they not only specify a POMDP, but they also specify a partial policy on that POMDP, by limiting the actions that can be taken in a given state. For example, in the voting program (Figure \ref{single-vote}), if an agent decides to stop asking the crowd for more votes, it can only return the answer that received more votes. It is not allowed to return the answer that received fewer votes. We imagine a non-trivial extension of \clowder\ that allows users to specify \emph{soft constraints}, in the same way that UCT allows users to specify a rollout policy to guide search. In this framework, \clowder\ would be able to deviate from the program. For example, instead of always returning the answer with more votes, it might return the answer with fewer votes, because maybe the user did not foresee that sometimes the answer with fewer votes is more likely to be correct.

4) \clowder\ assumes that all the primitives that use the same variable specify the same state space for that variable. Such a restriction makes life more difficult for non-experts. In particular, this assumption  may lead to unforeseen crashing or unexpected behavior. However, we envision at least two methods for ameliorating these scenarios. The first is to type our language, making it impossible to write programs that would crash the compiler or planner. The second is to use polymorphic typing or subtyping so that primitives are more flexible. 

5) Solving large POMDPs is a hard problem, and \clowder\ creates large POMDPs. The scalability of our system is a weakness that we hope to address. One way to reduce the size of the POMDPs that \clowder\ creates is to use state abstraction. If we can analyze the programs to determine the states that are irrelevant for making decisions, we can eliminate them and significantly increase the size of the programs that we can write.  

6) While a procedural Lisp-like language is easy for us to compile and interpret, we believe that  most users prefer a more imperative C-like language. Converting the \clowder\ language to one with a more familiar syntax should increase usability and adoption.

7) Finally, \clowder\ assumes that experts either know, or can write down the models for their primitives. However, we can easily extend the language to allow experts to direct \clowder\ to learn the models using reinforcement learning, thereby reducing the amount of work that experts need to put into the system.


\bibliographystyle{abbrv}
\bibliography{LinMausamWeld}

\begin{thebibliography}{10}

\bibitem{andre-nips01}
D.~Andre and S.~J. Russell.
\newblock Programmable reinforcement learning agents.
\newblock In {\em NIPS}, 2001.

\bibitem{andre-aaai02}
D.~Andre and S.~J. Russell.
\newblock State abstraction for programmable reinforcement learning agents.
\newblock In {\em AAAI}, 2002.

\bibitem{barto-ai95}
A.~G. Barto, S.~J. Bradtke, and S.~P. Singh.
\newblock Learning to act using real-time dynamic programming.
\newblock {\em Artificial Intelligence}, 72:81--138, 1995.

\bibitem{bernstein10}
M.~S. Bernstein, G.~Little, R.~C. Miller, B.~Hartmann, M.~S. Ackerman, D.~R.
  Karger, D.~Crowell, and K.~Panovich.
\newblock Soylent: A word processor with a crowd inside.
\newblock In {\em UIST}, 2010.

\bibitem{boutilier-aaai00}
C.~Boutilier, R.~Reiter, M.~Soutchanski, and S.~Thrun.
\newblock Decision-theoretic, high-level agent programming in the situation
  calculus.
\newblock In {\em AAAI}, 2000.

\bibitem{chajewska-aaai00}
U.~Chajewska, D.~Koller, and R.~Parr.
\newblock Making rational decisions using adaptive utility elicitation.
\newblock In {\em AAAI}, 2000.

\bibitem{dai-aij13}
P.~Dai, C.~H. Lin, Mausam, and D.~S. Weld.
\newblock Pomdp-based control of workflows for crowdsourcing.
\newblock {\em Artificial Intelligence}, 202:52--85, 2013.

\bibitem{dai-aaai10}
P.~Dai, Mausam, and D.~S. Weld.
\newblock Decision-theoretic control of crowd-sourced workflows.
\newblock In {\em AAAI}, 2010.

\bibitem{dietterich-jair00}
T.~G. Dietterich.
\newblock Hierarchical reinforcement learning with the maxq value function
  decomposition.
\newblock {\em Journal of Artificial Intelligence Research}, 13:227--303, 2000.

\bibitem{givan-aij03}
R.~Givan, T.~Dean, and M.~Greig.
\newblock Equivalence notions and model minimization in markov decision
  processes.
\newblock {\em Artificial Intelligence}, 147:163--223, 2003.

\bibitem{kamar-aamas12}
E.~Kamar, S.~Hacker, and E.~Horvitz.
\newblock Combining human and machine intelligence in large-scale
  crowdsourcing.
\newblock In {\em AAMAS}, 2012.

\bibitem{lin-aaai12}
C.~H. Lin, Mausam, and D.~S. Weld.
\newblock Dynamically switching between synergistic workflows for
  crowdsourcing.
\newblock In {\em AAAI}, 2012.

\bibitem{little09}
G.~Little, L.~B. Chilton, M.~Goldman, and R.~C. Miller.
\newblock Turkit: tools for iterative tasks on mechanical turk.
\newblock In {\em KDD-HCOMP}, pages 29--30, 2009.

\bibitem{marthi-ijcai05}
B.~Marthi, S.~Russell, D.~Latham, and C.~Guestrin.
\newblock Concurrent hierarchical reinforcement learning.
\newblock In {\em IJCAI}, 2005.

\bibitem{mcallester-99}
D.~McAllester.
\newblock Bellman equations for stochastic programs, 1999.

\bibitem{milnewitten08}
D.~Milne and I.~H. Witten.
\newblock Learning to link with wikipedia.
\newblock In {\em Proceedings of the ACM Conference on Information and
  Knowledge Management}, 2008.

\bibitem{parr-nips98}
R.~Parr and S.~Russell.
\newblock Reinforcement learning with hierarachies of machines.
\newblock In {\em NIPS}, 1998.

\bibitem{pinto-icmla10}
J.~Pinto, A.~Fern, T.~Bauer, and M.~Erwig.
\newblock Robust learning for adaptive programs by leveraging program
  structure.
\newblock In {\em ICMLA}, 2010.

\bibitem{ratinov11}
L.~Ratinov, D.~Roth, D.~Downey, and M.~Anderson.
\newblock Local and global algorithms for disambiguation to wikipedia.
\newblock In {\em Proceedings of the Annual Meeting of the Association of
  Computational Linguistics}, 2011.

\bibitem{sanner11}
S.~Sanner.
\newblock Relational dynamic influence diagram language (rddl): Language
  description.
\newblock Technical report, NICTA and the Australian National University, 2011.

\bibitem{silver-nips10}
D.~Silver and J.~Veness.
\newblock Monte-carlo planning in large pomdps.
\newblock In {\em NIPS}, 2010.

\bibitem{simpkins-oopsla08}
C.~Simpkins, S.~Bhat, C.~I. Jr., and M.~Mateas.
\newblock Towards adaptive programming: Integrating reinforcement learning into
  a programming language.
\newblock In {\em OOPSLA}, 2008.

\bibitem{weld-hcomp11}
D.~S. Weld, Mausam, and P.~Dai.
\newblock Human intelligence needs artificial intelligence.
\newblock In {\em HCOMP}, 2011.

\bibitem{younes-ipc04}
H.~L.~S. Younes and M.~L. Littman.
\newblock Ppddl1.0: The language for the probabilistic part of ipc-4.
\newblock In {\em IPC}, 2004.

\end{thebibliography}
\appendix
We have shown how to write a program that polls workers in order to find the best answer to some question. However, requesters can do better by asking the question in multiple ways \cite{lin-aaai12}. Figure \ref{multiple-vote} shows how to write a voting program if you have two methods for asking the same question.

While the primary goal of \clowder\ is to enable non-experts to write POMDPs, experts can also use \clowder\ to quickly build large and complex POMDPs. Figure \ref{rocksample} shows how one can use \clowder\ to write a goal-based \emph{rocksample}. We note that the program we write constrains the possible policies to ones that are more likely to be optimal (though it may not include the most optimal policy).

\begin{figure}[H]
\begin{tabbing}
\noindent \texttt{(d}\=\texttt{efine} \texttt{(m-vote q0 q1 a0 a1 c0 c1)}\\
\> \texttt{(c}\=\texttt{hoose}\\
\>\>\texttt{(if} \= \texttt{(crowd-vote q0 a0 a1))}\\
\>\>\>\texttt{(m-vote q0 q1 a0 a1 (+ c0 1) c1)}\\
\>\>\>\texttt{(m-vote q0 q1 a0 a1 c0 (+ c1 1)))}\\
\>\>\texttt{(if} \= \texttt{(crowd-vote q1 a0 a1))}\\
\>\>\>\texttt{(m-vote q0 q1 a0 a1 (+ c0 1) c1)}\\
\>\>\>\texttt{(m-vote q0 q1 a0 a1 c0 (+ c1 1)))}\\
\>\>\texttt{(if (> c0 c1) \#t \#f)))}\\
\end{tabbing}
\vspace{-20px}
\caption{A \clowder\ program for multiple workflows. \texttt{q0,q1} are the two ways of asking the same question, \texttt{a0,a1} are the two possible answers, and \texttt{c0,c1} count the number of votes for each choice.}
\label{multiple-vote}
\end{figure}
\vspace{-10px}
\begin{figure}[H]
\begin{tabbing}
\noindent \texttt{(d}\=\texttt{efine (move start end)}\\
\> \texttt{(if}\=\ \texttt{(= start end)} \\
\>\> \texttt{end}\\
\>\> \texttt{(choose}\= \texttt{(move (move-north start) end)}\\
\>\>\>\texttt{(move (move-south start) end)}\\
\>\>\>\texttt{(move (move-east start) end)}\\
\>\>\>\texttt{(move (move-west start) end))))}\\
\end{tabbing}
\begin{tabbing}
\noindent \texttt{(d}\=\texttt{efine}\texttt{ (r-s pos rocks exit-pos)}\\
\>    \texttt{(c}\=\texttt{hoose}\\
\>\> \texttt{(move pos exit-pos)}\\
\>\> \texttt{(let} \= \texttt{ ((good-rock (find-good-rock rocks)))}\\
\>\>\> \texttt{(r-s }\= \texttt{(sample (move pos good-rock))}\\
\>\>\>\> \texttt{(remove good-rock rocks) }\\
\>\>\>\> \texttt{exit-pos))))}\\
\end{tabbing}
\vspace{-20px}
\caption{A \clowder\ program for \emph{rocksample}. \texttt{pos} is the initial position, \texttt{rocks} is a list of rocks, and \texttt{exit-pos} is the exit position.}
\label{rocksample}
\end{figure}

\texttt{move-*} and \texttt{sample} are the primitives that need to expertly-defined. These definitions boostrap the creation of the POMDP. For example, The \clowder\ value of each rock can be defined by the primitives as a pair, where the first element is a binary indicator of whether or not the rock is good, and the second element is the position of the rock. The Normal value of each rock can also be a pair, where the first element is a rock id and the second element is its position. Then, the behavior of the \clowder\ values of the return values of these primitives (and thereby all subexpressions that use the return values), are given by the expert-defined transition probabilities. \texttt{find-good-rock} is a user-defined function (not shown) that can be viewed as a generalization of the voting program. 

\texttt{r-s} is the \emph{rocksample} program. It contains two choice points. The first choice is to simply move to the exit. The second choice is to first find a good rock, move to where the good rock is, sample it (a primitive), remove the rock from our list of rocks, and recurse.  

To define the policy, we can write a goal elicitation module that asks the user how many rocks should be sampled before quitting. For example, if the user specifies that all rocks be sampled, then the agent should find the expected minimum cost policy to sample all the rocks, where the costs are given by the primitives.

\end{document}